\documentclass[letterpaper]{article}
\usepackage{uai2020}
\usepackage[margin=1in]{geometry}
\usepackage{nicefrac}
\usepackage{times}
\usepackage{enumerate}
\usepackage{microtype}
\usepackage{hyperref}
\usepackage[ruled,linesnumbered]{algorithm2e}
\usepackage{amssymb}
\usepackage{amsmath}

\usepackage[ruled,linesnumbered]{algorithm2e}
\usepackage{amssymb}
\usepackage{amsmath}

\usepackage{amsthm}

\newtheorem{prop}{Proposition}
\theoremstyle{definition}
\newtheorem{defn}[prop]{Definition}
\newtheorem{example}[prop]{Example}

\DeclareMathOperator{\dep}{D}

\DeclareMathOperator{\pa}{pa}
\newcommand\indep{\protect\mathpalette{\protect\independenT}{\perp}}
\def\independenT#1#2{\mathrel{\rlap{$#1#2$}\mkern4mu{#1#2}}}

\usepackage{xparse}
\usepackage{tkz-graph}
\usetikzlibrary{shapes.geometric}
\newsavebox{\fminipagebox}
\NewDocumentEnvironment{fminipage}{m O{\fboxsep}}
 {\par\kern#2\noindent\begin{lrbox}{\fminipagebox}
  \begin{minipage}{#1}\ignorespaces}
 {\end{minipage}\end{lrbox}%
  \makebox[#1]{%
    \kern\dimexpr-\fboxsep-\fboxrule\relax
    \fbox{\usebox{\fminipagebox}}%
    \kern\dimexpr-\fboxsep-\fboxrule\relax
  }\par\kern#2
 }

\usepackage{subcaption}

\usepackage{natbib}
\renewcommand{\cite}{\citep}
\title{Measurement Dependence Inducing Latent Causal Models}
\author{ 
  Alex Markham\footnotemark% \thanks{Footnote for author to give an
and  Moritz Grosse-Wentrup\footnotemark[1]\footnotemark[2]\footnotemark[3] \\
\textsuperscript{1}\footnotetext{test12}Research Group Neuroinformatics, Faculty of Computer Science, University of Vienna\\
\textsuperscript{2}Research Platform Data Science @ Uni Vienna  \\
\textsuperscript{3}Vienna Cognitive Science Hub \\
}

\begin{document}
\maketitle

\begin{abstract}
  We consider the task of causal structure learning over measurement dependence inducing latent (MeDIL) causal models. % rather latent variable identification?
  We show that this task can be framed in terms of the graph theoretic problem of finding edge clique covers, resulting in an algorithm for returning minimal MeDIL causal models (minMCMs).
  This algorithm is non-parametric, requiring no assumptions about linearity or Gaussianity. 
  Furthermore, despite rather weak assumptions about the class of MeDIL causal models, we show that \emph{minimality} in minMCMs implies some rather specific and interesting properties.
  By establishing MeDIL causal models as a semantics for edge clique covers, we also provide a starting point for future work further connecting causal structure learning to developments in graph theory and network science. 
\end{abstract}

\section{INTRODUCTION}
\let\oldfootnotesize\footnotesize
\renewcommand*{\footnotesize}{\oldfootnotesize\fontsize{8pt}{8pt}\selectfont}

Despite the many theoretical and practical difficulties, establishing and understanding causal relationships remains one of the fundamental goals of scientific research.
Consequently, many different approaches have been developed, with applications spanning a diverse range of fields, e.g., from epidemiology to psychometrics to neuroimaging \cite{Parascandola_2001,Hoover_2006,Seth_2015}. %,Pearl:2000,Spirtes:2000
Some of the most well-known approaches include Granger causality \cite{Granger_1969} for time-series data, the Rubin causal model and potential outcomes framework \cite{Holland:1986} for randomized controlled trials, and functional causal models and the representation of their causal structure as directed acyclic graphs \cite{Pearl:2000,Spirtes:2000}. % \cite{Pearl:1995,Pearl:2000,Spirtes:2000,Spirtes:2009}. Rubin:1974,Rubin:2005,  ,Granger_2004
The last of these, the directed acyclic graph (DAG), provides the context for our approach to causal structure learning.

Roughly speaking, causal structure learning (CSL) typically focuses on identifying which variables are directly causally related and how these \textit{direct causal relations form a structure} over which indirect causal relations exist.
One way of characterizing CSL algorithms is according to which of the three following assumptions they rely on:
(i) the \textit{causal Markov assumption}, which says the random variables are (conditionally) independent (denoted by \(\indep\)) if the corresponding vertices in the DAG are d-separated (denoted by \(\perp\));
(ii) the \textit{causal faithfulness assumption}, which says the vertices in the DAG are d-separated if the corresponding random variables are (conditionally) independent;
and (iii) the \textit{causal sufficiency} of the set of variables, i.e. that there are no unobserved or latent common causes.
The basic approach to CSL---namely the original constraint-based IC and PC algorithms \cite{verma1990equivalence,Spirtes_1991}---rely on all three, while many of the algorithms developed in the 30 years since (as we will see in Section \ref{sec:related-work}) relax these assumptions.

Considering applications of CSL to, for example, psychometrics and neuroimaging, the assumption of causal sufficiency seems implausible.
For a data set consisting solely of answers to a depression diagnostic questionnaire or of voxel intensities in calcium imaging recordings (with random variables corresponding respectively to the questions or voxels), we think it is relatively uncontroversial to claim that not only are the random variables not causally sufficient, but indeed \textit{every} dependence relation among them is induced by unobserved latent variables (respectively either cognitive processes related to, e.g., depression, or calcium signaling in cellular tissue, plus other confounders).
In fields and applications such as these---where interventions are often difficult or unfeasible, and where the goal is to reason about underlying causes based on their measurable effects---a more tailored causal modeling framework may prove insightful.
Thus, the main difference between the traditional approach outlined above and the one we present in this paper is that we assume a strong causal \textit{in}sufficiency of the random variables being modeled and therefore are able to represent a different (but not entirely disjoint) class of causal structures than is possible with DAGs in the traditional approach.

The rest of the paper is organized as follows:
We begin by reviewing related work, emphasizing points of departure.
In Section \ref{sec:obs_mcm} we define measurement dependence inducing latent (MeDIL) causal models to be the class of latent measurement models in which measurement variables can only be effects (and not causes---contrary to the definition of measurement models explored by others), making no further assumptions about linearity or parametrizations of the distributions.
We then introduce the notions of observational consistency and minimality, allowing us to, for a given (estimated) distribution of measurement variables, construct a minimal MeDIL causal model (minMCM).
Then, in Section \ref{sec:ecc}, by framing minMCMs as edge clique covers (ECCs) of the undirected dependency graph over measurement variables, we note how two notions of minimality emerge.
Subsequently, despite our nonrestrictive assumptions and notion of minimality in minMCMs, we are able to prove
(i) that a minMCM lower bounds the number of latent variables or the number of functional causal relations (depending on which notion of minimality is used),
(ii) that the latent variables of the minMCM are all pairwise independent, and
(iii) that (somewhat surprisingly) the minMCM can have more latent causes than measured variables.
In Section \ref{sec:alg} we describe an algorithm for learning minMCMs from only \textit{unconditional} (in)dependencies.
Finally, we demonstrate our approach with an application to a psychometric data set in Section \ref{sec:application}, before concluding with a discussion of promising directions for future work.

\subsection{RELATED WORK}
\label{sec:related-work}

Elaborating on the basic approach mentioned above, CSL without latents amounts to finding an \textit{essential graph} \cite{andersson1997characterization}, a mixed graph with directed and undirected edges, which represents the Markov equivalence containing the true DAG.
The essential graph is typically found by using either a score- or constraint-based approach.
Score-based methods find an essential graph by directly optimizing a score of how well it fits the data samples \cite{chickering2002optimal}.% ,ramsey2015scaling,zheng2018dags,lachapelle2019gradient}.
Constraint-based methods take a set of conditional independence relations as input (which must be estimated or acquired somehow before applying the algorithm), and these relations constitute a set of constraints on the possible d-separations, which the output essential graph satisfies \cite{verma1990equivalence,Spirtes_1991}.
Our approach in this paper is more closely related to constraint-based methods, especially their extensions to latent variable models.

Extensions of CSL to causal models including latent variables (i.e., relaxing the causal sufficiency assumption), such as the FCI algorithm and its variants \cite{spirtes1999algorithm}, correspondingly extend the search space from essential graphs to partial ancestral graphs (PAGs), which have an additional three edge types (so five total), allowing them to represent the extended Markov equivalence class containing dependencies induced by latent variables. % ,ogarrio2016hybrid

In these terms, our latent CSL algorithm \textit{is not} searching for a PAG.
As we explain in sections \ref{sec:obs_mcm} and \ref{sec:ecc}, by making use of the strong causal insufficiency in this application space, we can directly represent the conditional independence constraints that form the input for our algorithm as an undirected dependence graph (UDG).
This UDG is essentially a PAG with only bidirected edges.
Or, put another way, it is a modified Markov random field \cite{Kindermann_1980} where the conditional independence relations are determined from the undirected edges by using strong causal insufficiency (see Proposition \ref{prop:det}) instead of the Markov property, thereby allowing the UDG to represent latent induced dependence (which Markov random fields are usually incapable of representing).

With the conditional independence constraints input in the form of a UDG over measurement variables, our algorithm essentially adds the latent causes and directed edges necessary to construct the minimally causally sufficient DAG containing latent and measurement variables.
Thus, instead of doing CSL in the presence of latent variables as is the case with FCI and similar algorithms, we \textit{use CSL to reason about latent variables}.

Our approach is more related in this respect to other work on measurement models \cite{silva2005generalized,Silva:2006,Kummerfeld:2014,Kummerfeld:2016}.
However, these other approaches utilize properties of the covariance matrix of the measurement variables, such as vanishing tetrad constraints, while we utilize graph theoretic properties of the UDG representation of conditional independencies.
This results in connections between our approach and causal feature learning \cite{Chalupka_2016} and causal consistency and abstraction \cite{Rubenstein:2017,beckers2019abstracting}, which will be discussed more with respect to future work in Section \ref{sec:sign-extens-furth}.
Another closely related approach is factor analysis, especially when framed in terms of using the topology of a Bayesian network of observed variables to reason about hidden factors \cite{martin1994discrete}, with the main difference being our goal of a minimally causally sufficient DAG as opposed to a statistically convenient (but not necessarily as causally relevant) factor model.

Overall, our approach has several points of overlap in terms of motivations and formal methods in existing CSL, measurement model, and factor analysis approaches.
However, we address the problem from a different perspective, utilizing the causal insufficiency property of our application space and graph theoretic edge clique cover methods to produce a novel algorithm.

% Need to discuss and cite \cite{martin1994discrete} and 

\section{MINIMAL MEDIL CAUSAL MODELS}
\label{sec:obs_mcm}

We begin with a formal definition of \textit{measurement dependence inducing latent (MeDIL) causal models}, before discussing the notion of observational consistency and its implications about minimality in such models.  

We use functional causal models (FCMs) to describe causal relations in complex systems.

\begin{defn}[Functional Causal Model]
  A \emph{functional causal model} is a triple \(\mathcal{M} = \langle \mathbf{V}, \mathbf{F}, \boldsymbol{\epsilon}\rangle\), where
  \begin{itemize}
  \item \(\mathbf{V}\) is the set of (endogenous) random variables, 
  \item \(\mathbf{F}\) is a set of functions defining each endogenous variable as a function of its direct causes (i.e., parents or \(\pa()\)) and its corresponding exogenous random variable, so that for each \(V_i \in \mathbf{V}\), we have \(V_i := f_i(\pa(V_i), \epsilon_i)\).
    Furthermore, \(\mathbf{F}\) is constrained such that no \(V_i\) is a direct cause of itself or any of its causes, removing the possibility of causal cycles.
  \item \(\boldsymbol{\epsilon}\) defines a joint probability distribution over the exogenous (or noise) variables, with a corresponding \(\epsilon_i \in \boldsymbol{\epsilon}\) for each \(V_i \in \mathbf{V}\), and with \(\epsilon_i\) being independent with \(\epsilon_j\) for each \(\epsilon_i, \epsilon_j \in \boldsymbol{\epsilon}\)
  \end{itemize}\qed    
\end{defn}

In particular, we are interested in latent CSL over measurement variables, so it is advantageous to move from the general FCM definition to a specifically structural/graphical definition that conceptually differentiates the set of endogenous variables into causally effective latent variables and their observed measurements, leading to the idea of MeDIL causal models:  

\begin{defn}[Measurement Dependence Inducing Latent Causal Model (MCM)]
  \label{defn:mm}
  A graphical MCM is a DAG, given by the triple \(\mathcal{G} = \langle \mathbf{L}, \mathbf{M}, \mathbf{E} \rangle\).
  \(\mathbf{L}\) and \(\mathbf{M}\) are disjoint sets of vertices, while \(\mathbf{E}\) is a set of directed edges between these vertices, subject to the following constraints:
  \begin{enumerate}
  \item all vertices in \(\mathbf{M}\) have in-degree of at least 1 and out-degree of 0
  \item all vertices in \(\mathbf{L}\) have out-degree of at least 1
  \item \(\mathbf{E}\) contains no cycles
  \end{enumerate}\qed
\end{defn}

There are no further constraints as to the variety of distributions and functional causal relations that MCMs can represent, i.e., they are non-parametric and their arrows can represent arbitrary functional relations between variables.
The formal constraints 1.~and 2.~in Definition \ref{defn:mm} are to ensure that MCMs are applicable to settings in which we can explicitly separate into disjoint sets the measured effect variables \(\mathbf{M}\) whose probabilistic dependencies must therefore be mediated by latent causes \(\mathbf{L}\).

However, the explicit separation of cause and effect and the corresponding latent structure in MCMs introduces its own difficulties for inference.
Namely, many latent models are consistent with a given probability distribution over observed effects, making the task of inferring a single latent model ill-posed.
% This is usually addressed by introducing strong assumptions, e.g., linearity (see Section \ref{sec:related-work}).
% We, however, want to study what we can learn from CSL over MCMs without such assumptions.
In order to help explain this consistency of different latent models and illustrate our strategy for restricting the problem so that inference is well-posed, consider the following definition and example.

\begin{defn}[Observational Consistency]
  \label{defn:oc}
  A MCM is \textit{observationally consistent} with a probability distribution over measurement variables if it is capable of inducing the pairwise dependencies (which can estimated from samples) of that distribution.
  % Furthermore, two MCMs are observationally consistent with each other if there exists a distribution of measurements that they are both capable of inducing.
  This can be seen as a weakening of the notion of observational equivalence corresponding to our extension from DAGs containing only observed variables to the notion of MCMs.%
  \footnote{\textit{observational} or \textit{Markov equivalence} \cite[pp. 16--20]{Pearl:2000} means two DAGs have the same skeletons and colliders, while observational consistency means that two MCMs have the same undirected dependency graphs over measurement variables (e.g., Figure \ref{fig:ex_oc})}%
\end{defn}

\begin{example}[Observational Consistency]
  \label{ex:obs}
  Suppose we have data consisting of peoples' answers to a questionnaire with four questions designed to measure depression and stress. 
  We assume that the answer to one question cannot cause the answer to another and therefore that the observed answers as well as any observed association between answers are the result of latent causes, such as depression or stress.
  Define random variables \(\mathbf{M} = \{M_1, M_2, M_3, M_4\}\) corresponding to answers to the four questions, and let them have only the following two pairwise independencies:
  \[M_1 \indep M_4 \quad\text{and}\quad M_2 \indep M_4\]
  
  The pairwise dependency structure between variables in \(\mathbf{M}\) is shown in Figure \ref{fig:ex_oc:a}, and three observationally consistent MCMs are shown in \ref{fig:ex_oc:b}, \ref{fig:ex_oc:c}, \ref{fig:ex_oc:d}.
  As this example demonstrates, multiple latent models can give rise to the same set of observed dependencies.

\qed
\end{example}
  
\tikzstyle{VertexStyle} = [shape = circle, minimum width = 2ex, draw] 
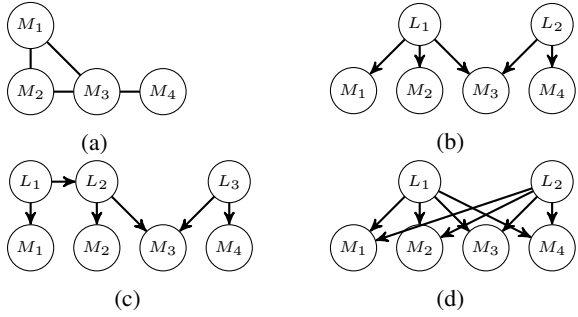
\begin{figure}[h!]
  \centering{\tiny
    \begin{minipage}{.26\textwidth}
      \subcaptionbox{\label{fig:ex_oc:a}}{
        \begin{tikzpicture}[scale=.4]
          \centering
          \SetGraphUnit{2.2}
          \begin{scope}[execute at begin node=$, execute at end node=$]
            \Vertex{M_1} \SO(M_1){M_2} \EA(M_2){M_3} \EA(M_3){M_4}
            \Edges(M_3, M_1, M_2, M_3, M_4)
          \end{scope}
        \end{tikzpicture}}
    \end{minipage}%
    \tikzstyle{EdgeStyle} = [->,>=stealth']%
    \begin{minipage}{.26\textwidth}
      \subcaptionbox{\label{fig:ex_oc:b}}{
        \begin{tikzpicture}[scale=.4]
          \centering
          \SetGraphUnit{2.2}
          \begin{scope}[execute at begin node=$, execute at end node=$]
            \Vertex{M_1} \EA(M_1){M_2} \EA(M_2){M_3} \EA(M_3){M_4}
            \NO(M_2){L_1} \NO(M_4){L_2}
            \Edges(L_1, M_1) \Edges(L_1, M_2) \Edges(L_1, M_3)
            \Edges[style=-](L_2, M_3) \Edges(L_2, M_4)  
          \end{scope}
        \end{tikzpicture}}
    \end{minipage}
    \begin{minipage}{.26\textwidth}
      \subcaptionbox{\label{fig:ex_oc:c}}{
        \begin{tikzpicture}[scale=.4]
          \centering
          \SetGraphUnit{2.2}
          \begin{scope}[execute at begin node=$, execute at end node=$]
            \Vertex{M_1} \EA(M_1){M_2} \EA(M_2){M_3} \EA(M_3){M_4}
            \NO(M_1){L_1} \NO(M_2){L_2} \NO(M_4){L_3}
            \Edges(L_1, M_1) \Edges(L_1, L_2)
            \Edges(L_2, M_2) \Edges(L_2, M_3)
            \Edges(L_3, M_3) \Edges(L_3, M_4)
          \end{scope}
        \end{tikzpicture}}
      \end{minipage}%
      \begin{minipage}{.26\textwidth}
        \subcaptionbox{\label{fig:ex_oc:d}}{
          \begin{tikzpicture}[scale=.4]
            \tikzstyle{EdgeStyle} = [->,>=stealth']%
            \centering
            \SetGraphUnit{2.2}
            \begin{scope}[execute at begin node=$, execute at end node=$]
              \Vertex{M_1} \EA(M_1){M_2} \EA(M_2){M_3} \EA(M_3){M_4}
              \NO(M_2){L_1} \NO(M_4){L_2}
              \Edges(L_1, M_1) \Edges(L_1, M_2) \Edges(L_1, M_3) \Edges(L_1, M_4)
              \Edges(L_2, M_1) \Edges(L_2, M_2) \Edges(L_2, M_3) \Edges(L_2, M_4)
            \end{scope}
          \end{tikzpicture}}
        \end{minipage}}
      \caption{(a) undirected dependency graph over \(\mathbf{M}\)---notice two missing edges corresponding to independencies; (b) minimal MCM over \(\mathbf{M}\); (c) non-minimal MCM observationally consistent with \(\mathbf{M}\); (d) MCM corresponding to ICA or FA}
  \label{fig:ex_oc}
\end{figure}

% One way of addressing this problem is by assuming linear dependencies and then selecting an appropriate subset of observed variables whose covariance matrices can be used to identify a unique number and structure of latent variables \cite{silva:2002,Silva:2006}.
% We take a different approach, by making no such linearity assumptions and instead
We address this problem by employing Ockham's razor to pick a \textit{minimal MCM (minMCM)} (e.g., Figure \ref{fig:ex_oc:b}).
  
\begin{defn}[Minimal MeDIL causal model (minMCM)]
  \label{defn:mlm}
  A \textit{minMCM} for a set of measurement variables \(\mathbf{M}\) is any least expressive (i.e., minimal) MCM that is observationally consistent with \(\mathbf{M}\).
  As \citet{Pearl:1995} note, a latent causal model's expressive power can be measured by the (in)dependencies it induces over the measured variables, with more dependencies corresponding to more expressive power. 
  In our case, criteria can be given for minimality in modified terms of the causal faithfulness and causal Markov assumptions:
  \begin{enumerate}
  \item in addition to being observationally consistent with its set of measurements, a minMCM must graphically induce the measurements without violating faithfulness; the notion of faithfulness used here is concerned with conditional independencies only over measurements and not all variables in the MCM, so we call it \textit{measurement-faithfulness}; note that Figure \ref{fig:ex_oc:b} is faithful to the conditional independencies in Example \ref{ex:obs} while Figure \ref{fig:ex_oc:d} is not---the MCM in Figure \ref{fig:ex_oc:b} is minimal while that in \ref{fig:ex_oc:d} is not
  \item considering arbitrary subsets of the latents, \(\mathbf{Z} \subseteq \mathbf{L}\), there are as few d-separations of the form \(M_i \not\perp M_j \mid \mathbf{Z}\) as (faithfully) possible, i.e., such d-separations only exist in an minMCM if implied by the (in)dependencies and causal insufficiency of the distribution only over measurement variables; we call this \textit{measurement-Markov} since it says the only d-separations in the minMCM are those implied by measurement-faithfulness\footnote{just as is the case with the usual causal faithfulness and Markov conditions}; note that Figure \ref{fig:ex_oc:c} does not satisfy this\qed
    % \footnote{Note that minimality here mostly comes from the measurement-Markov criteria, and so non-minimal MCMs can be seen as relaxations of this criteria.}
  \end{enumerate}
\end{defn}

% Thus, instead of using the (linear) covariance matrix itself as in \cite{Silva:2006}, finding a minMCM only requires independencies, which can be estimated from nonlinear measures of dependence.
% Estimating a set of independencies from the data before attempting CSL is also necessary for other methods, however our method needs only to consider \textit{unconditional} independencies (unlike PC and IC, which require conditional independencies \cite{Spirtes_1991,verma1990equivalence}) between measurement variables, greatly reducing the number of independencies that must be estimated.
Learning a minMCM for a data set only requires considering the \textit{unconditional} independence relations among its variables, unlike the other methods mentioned in Section \ref{sec:related-work}.
% The sufficiency of only considering unconditional independencies for our method follows 
This follows from Proposition \ref{prop:det}.

\begin{figure*}[h!]
    \small\centering
    \begin{minipage}{\textwidth}
\fbox{  \subcaptionbox{\label{fig:diff:b}}{
    \begin{minipage}{.964\textwidth}
      \centering
  \begin{tikzpicture}[scale=.4]
    \SetGraphUnit{4}
    \begin{scope}[execute at begin node=$, execute at end node=$]\tiny
      \Vertex{M_1} \EA(M_1){M_2} \EA(M_2){M_3} \EA(M_3){M_4} \EA(M_4){M_5} \EA(M_5){M_6} \EA(M_6){M_7} \EA(M_7){M_8}
      \NO(M_1){C_1} \NO(M_4){C_3} \NO(M_5){C_4} \NO(M_6){C_5}
      
      \tikzstyle{VertexStyle} = [circle, dashed, draw=red, thick]
      \NO(M_8){C_7} \NO(M_2){C_2}

      \tikzstyle{VertexStyle} = [circle, dotted, draw=blue, thick]
      \NO(M_7){C_6}       
      
      \Edges(C_1, M_1)  \Edges(C_1, M_5) % \Edges(C_1, M_2) \Edges(C_1, M_3)
      % \Edges(C_2, M_1) \Edges(C_2, M_2) \Edges(C_2, M_3) \Edges(C_2, M_7)
      \Edges(C_3, M_1) \Edges(C_3, M_4) % \Edges(C_3, M_7)
      \Edges(C_4, M_2) \Edges(C_4, M_3) \Edges(C_4, M_5) \Edges(C_4, M_6) \Edges(C_4, M_8)
      \Edges(C_5, M_4) \Edges(C_5, M_6) % \Edges(C_5, M_8)
      % \Edges(C_6, M_2) \Edges(C_6, M_3) \Edges(C_6, M_7) \Edges(C_6, M_8)
      % \Edges(C_7, M_4) \Edges(C_7, M_7) \Edges(C_7, M_8)

      \foreach \from/\to in {C_2/M_1,C_2/M_2,C_2/M_3,C_2/M_7,C_7/M_4,C_7/M_7,C_7/M_8} \draw[dashed, color=red, line width=1pt, ->, >=stealth'] (\from) -> (\to);
      \foreach \from/\to in {C_1/M_2,C_1/M_3,C_3/M_7,C_5/M_8,C_6/M_2,C_6/M_3,C_6/M_7,C_6/M_8} \draw[dotted, color=blue, line width=1pt, ->, >=stealth'] (\from) -> (\to);
    \end{scope}
  \end{tikzpicture}
\end{minipage}
}}

\begin{minipage}{\linewidth}
\fbox{
\subcaptionbox{\label{fig:diff:a}}{
    \begin{minipage}[t][12.75em]{.3\linewidth}% {.925\textwidth}
    \centering
  \begin{tikzpicture}
    \SetGraphUnit{1.5}
    \begin{scope}[execute at begin node=$, execute at end node=$]\tiny
      \Vertices{circle}{M_1, M_2, M_3, M_4, M_5, M_6, M_7, M_8}
      \Edges(M_1, M_2, M_3, M_1, M_5, M_1, M_7, M_3, M_7, M_2, M_7, M_1, M_4, M_6, M_8, M_4, M_7, M_8)
      \Edges(M_5, M_6, M_8, M_5, M_3, M_5, M_2, M_6, M_3, M_2, M_8, M_3, M_7, M_2)
    \end{scope}
  \end{tikzpicture}
\end{minipage}}}%
  \tikzstyle{EdgeStyle} = [->,>=stealth']%
\begin{minipage}[b]{.64\linewidth}\centering
\fbox{ \subcaptionbox{\label{fig:diff:c}}{
    \begin{minipage}{\linewidth}\centering
  \begin{tikzpicture}[scale=.25]
    \SetGraphUnit{4}
    \begin{scope}[execute at begin node=$, execute at end node=$]\tiny
      \Vertex{M_1} \EA(M_1){M_2} \EA(M_2){M_3} \EA(M_3){M_4} \EA(M_4){M_5} \EA(M_5){M_6} \EA(M_6){M_7} \EA(M_7){M_8}
      \NO(M_1){C_1} \NO(M_4){C_3} \NO(M_5){C_4} \NO(M_6){C_5} \NO(M_7){C_6}

      \Edges(C_1, M_1) \Edges(C_1, M_2) \Edges(C_1, M_3) \Edges(C_1, M_5)
      \Edges(C_3, M_1) \Edges(C_3, M_4) \Edges(C_3, M_7)
      \Edges(C_4, M_2) \Edges(C_4, M_3) \Edges(C_4, M_5) \Edges(C_4, M_6) \Edges(C_4, M_8)
      \Edges(C_5, M_4) \Edges(C_5, M_6) \Edges(C_5, M_8)
      \Edges(C_6, M_2) \Edges(C_6, M_3) \Edges(C_6, M_7) \Edges(C_6, M_8)
    \end{scope}
  \end{tikzpicture}
\end{minipage}
}}

\fbox{  \subcaptionbox{\label{fig:diff:d}}{
    \begin{minipage}{\linewidth}\centering
  \begin{tikzpicture}[scale=.25]
    \SetGraphUnit{4}
    \begin{scope}[execute at begin node=$, execute at end node=$]\tiny
      \Vertex{M_1} \EA(M_1){M_2} \EA(M_2){M_3} \EA(M_3){M_4} \EA(M_4){M_5} \EA(M_5){M_6} \EA(M_6){M_7} \EA(M_7){M_8}
      \NO(M_1){C_1} \NO(M_2){C_2} \NO(M_4){C_3} \NO(M_5){C_4} \NO(M_6){C_5} \NO(M_8){C_7} 

      \Edges(C_1, M_1) \Edges(C_1, M_5)
      \Edges(C_2, M_1) \Edges(C_2, M_2) \Edges(C_2, M_3) \Edges(C_2, M_7)
      \Edges(C_3, M_1) \Edges(C_3, M_4)
      \Edges(C_4, M_2) \Edges(C_4, M_3) \Edges(C_4, M_5) \Edges(C_4, M_6) \Edges(C_4, M_8)
      \Edges(C_5, M_4) \Edges(C_5, M_6)
      \Edges(C_7, M_4) \Edges(C_7, M_7) \Edges(C_7, M_8)
    \end{scope}
  \end{tikzpicture}
\end{minipage}
}}
\end{minipage}
\end{minipage}
\end{minipage}
\caption{(a) MCM, where each \(C_i\) corresponds to a maximal clique in \(\dep(\mathbf{M})\)---dashed red edges/vertices are redundant for vertex-minimality while blue dotted edges/vertices are redundant for edge-minimality;
  (b) \(\dep(\mathbf{M})\)---undirected dependency graph of \(\mathbf{M} = \{M_1, \dots, M_8\}\);
  (c) vertex-minimal minMCM of \(\dep(\mathbf{M})\);
  (d) edge-minimal minMCM of \(\dep(\mathbf{M})\)}
  \label{fig:diff}
\end{figure*}
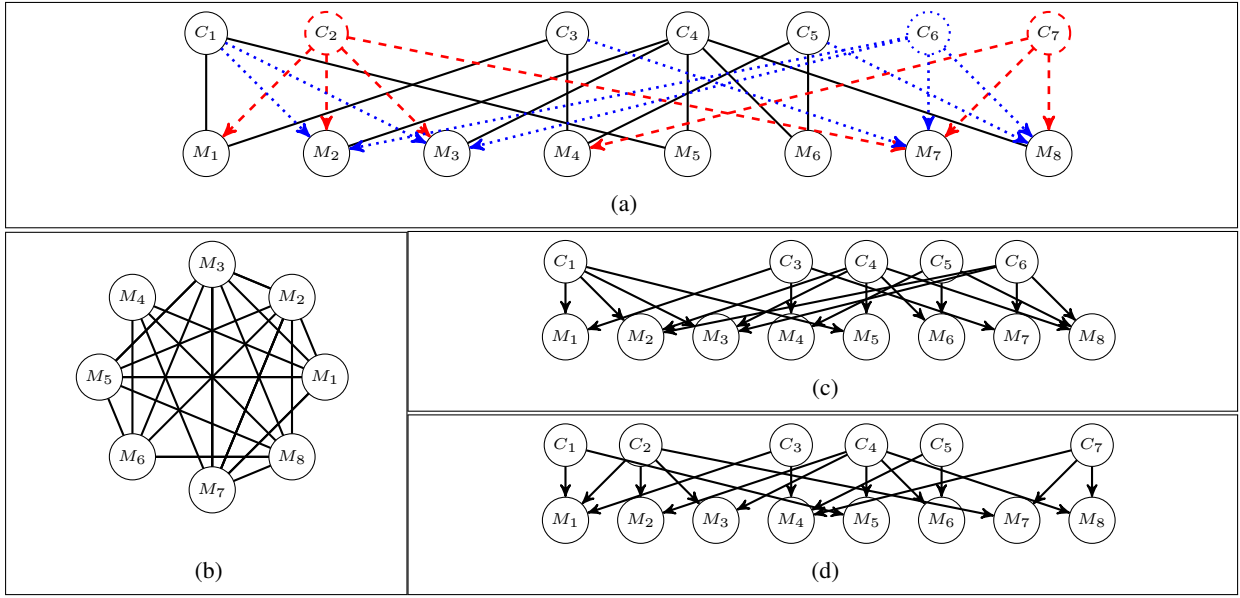

\begin{prop}
  \label{prop:det}
  In a MCM, the set of unconditional (in)dependencies over measurement variables fully determines the set of conditional (in)dependencies over measurement variables.
\end{prop}
\begin{proof}
  The Causal Markov and Causal Faithfulness assumptions (CMA and CFA, respectively) imply that two variables are probabilistically independent if and only if they are \(d\)-separated (allowing us to use independence/d-separation and \(\indep/\perp\) interchangeably).
  Recall from Definition \ref{defn:mm} that all dependence relations (and therefore, by the CMA and CFA, d-connections) between measurement variables are mediated by latent variables.
  Hence, all measurement variables have out-degree 0, and so any measurement variable in a path between two other measurement variables must be a collider and any dependent measurement variables must share at least one latent parent.
  This means that the set of unconditional (in)dependencies over measurement variables fully determines the set of conditional (in)dependencies as follows: for all \(M_i, M_j, M_k \in \mathbf{M}\), 
  \begin{itemize}
  \item \(M_i \not\indep M_j \implies M_i \not\indep M_j \mid M_k\)
  \item \(M_i \indep M_j \implies\\
    \begin{cases}
      M_i \indep M_j \mid M_k, & \text{ if } M_i \indep M_k \text{ or } M_j \indep M_k\\
      M_i \not\indep M_j \mid M_k, & \text{ otherwise} 
    \end{cases}\)
  \end{itemize}
  
\end{proof}
  
As we will see in Section \ref{sec:alg}, even though estimating conditional independencies is not required for our method, doing so nevertheless can help determine whether any of the assumptions have been violated.

\section{MINIMAL MEDIL CAUSAL MODELS AS EDGE CLIQUE COVERINGS}
\label{sec:ecc}

We can now present our main insight:

\begin{prop}
  \label{prop:mlm-ecc}
  The problem of finding a minMCM for a set of measurement variables can be framed as the graph theoretical problem of finding a \textit{minimum edge clique covering} (ECC)\footnote{A minimum ECC over an undirected graph is a collection of cliques that exactly covers its edges, where an edge \(E = (V_i, V_j)\) is covered by clique \(C\) iff  \(V_i, V_j \in C\).} \cite{Erdos_1966,Gramm_2009,Ennis_2012} over the corresponding undirected dependency graph of the measurement variables. %,scheinerman1999fractional
\end{prop}

\begin{proof}
  For a given set of measurement variables \(\mathbf{M}\), denote the \textit{undirected dependency graph} as  \(\dep(\mathbf{M})\), e.g., Figure \ref{fig:ex_oc:a}, where an edge represents dependence and the lack of an edge represents independence.
  Proposition \ref{prop:det} tells us that \(\dep(\mathbf{M})\), though it only encodes unconditional (in)dependencies, contains all necessary information for characterizing observationally consistent MCMs.
  Consider the MCM \(\mathcal{G} = \langle \mathbf{L}, \mathbf{M}, \mathbf{E}\rangle\) constructed from a set of cliques \(\mathbf{C}\) comprising a minimum ECC over \(\dep(\mathbf{M})\) using the following procedure:
  (i) posit a latent \(L_C \in \mathbf{L}\) iff \(C \in \mathbf{C}\) and (ii) posit a directed edge \(E \in \mathbf{E}\) from the latent \(L_C\) to the measurement variable \(M\) iff \(M \in C\).
  In other words, \(\mathbf{G}\) is a MCM with measurement variables \(\mathbf{M}\), one latent for each clique in the minimum ECC over \(\dep(\mathbf{M})\), and an edge from each latent to exactly the measurement variables in the corresponding clique.

  Note that \(\mathbf{G}\) is not only observationally consistent with \(\dep(\mathbf{M})\) but also captures its independencies and is thus faithful, satisfying criterion 1.~of Definition \ref{defn:mlm}.
  Furthermore, the construction of \(\mathbf{G}\) from a minimum ECC ensures that latents are only posited when necessitated by the dependencies between measurements, satisfying criterion 2.~of Definition \ref{defn:mlm}.
  Thus, \(\mathbf{G}\) is an minMCM for \(\dep(\mathbf{M})\).
  
\end{proof}

  A minimum ECC can be minimal in two related but distinct ways:
  the original and more well-studied approach seeks the smallest number of cliques needed to cover all edges (this is equivalent to the \textit{intersection number} \cite{Erdos_1966}), while another justifiable approach is to seek an ECC requiring the fewest assignments of vertices to cliques.
  The corresponding interpretation for minMCMs is vertex-minimal (fewer cliques imply fewer latents imply fewer total vertices) and edge-minimal (fewer assignments of measurement vertices to cliques implies fewer directed edges from latent to measurement vertices), resulting in Proposition \ref{prop:min}.
  There are some undirected dependency graphs for which the vertex-minimal and edge-minimal minMCMs are identical, such as figures \ref{fig:ex_oc} and \ref{fig:prop_6}, but this identity does not hold generally \cite{Ennis_2012} (see Figure \ref{fig:diff}).
  In either approach to minimality, the resulting minMCM induces the same set of dependencies over measurement variables and thus has the same expressive power (w.r.t. the measurement variables).
  We thus see no straightforwardly principled way of picking one approach over the other, and so we present both in hopes that practitioners will use whichever one (or both) they judge most sensible/interesting for their particular application.  

Regardless of which notion of minimality is used, minMCMs have some interesting properties.
  First, they lower bound (i) the number of causal concepts or (ii) the number of functional causal relations that are required to model measurements of a complex system at any level of granularity (Proposition \ref{prop:min}).
 Second, minMCMs contain no causal links between the latent variables (Proposition \ref{prop:indep}).
  Finally, in contrast to factor analysis, a minMCM may require more latent than measurement variables (Proposition \ref{prop:more}).

\begin{prop}
  \label{prop:min}
  For a given set of unconditional pairwise dependencies among measurement variables \(\mathbf{M}\), a minMCM gives a lower bound on the number of latent variables or edges (depending on the measure of minimality is used) required in any (faithful and observationally consistent) MCM.
\end{prop}
\begin{proof}
  This is a direct consequence of the construction of minMCMs from either the clique-minimum or assignment-minimum ECC of \(\dep(\mathbf{M})\), as described in Proposition \ref{prop:mlm-ecc}.
\end{proof}

\begin{prop}
  \label{prop:indep}
  In a minMCM, each latent variable is d-separated from every other latent variable.
\end{prop}

\begin{proof}
  Intuitively, this is a result of the definition of a minMCM being minimal in the sense of least expressive (and thus having as few latents or edges):
  if two latent variables are d-connected, then the dependencies among measurement variables that they induce could also instead be induced by a single latent variable (which also results in fewer edges).
  A minMCM has no redundant latent variables or edges and therefore no d-connected latent variables.
  For example, note in that the MCMs in figures \ref{fig:ex_oc:b} and \ref{fig:ex_oc:c} induce the same d-separations over the measurement variables, but that \ref{fig:ex_oc:b} with its d-separated latents has the fewer latents and fewer total edges.
  More formally, this follows directly from procedure for constructing an minMCM in Proposition \ref{prop:mlm-ecc} and Algorithm \ref{alg:obs}.
\end{proof}

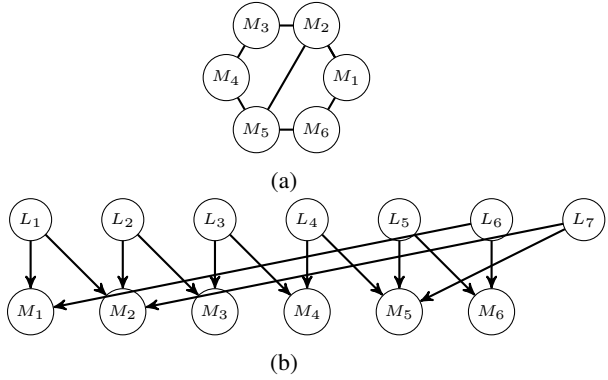
\begin{figure}[t]
  \tiny
  \tikzstyle{EdgeStyle} = [-]
  \subcaptionbox{\label{fig:prop_6:a}}{
    \begin{minipage}{.45\textwidth}  \centering
      \label{fig:prop_6:a}
      \begin{tikzpicture}[scale=1]
        \centering
      \SetGraphUnit{.8}
      \begin{scope}[execute at begin node=$, execute at end node=$]
        \Vertices{circle}{M_1, M_2, M_3, M_4, M_5, M_6}
        % \Vertex{1} \EA(1){2} \EA(2){3} \SO(3){4} \WE(4){5} \WE(5){6}
        \Edges(M_1, M_2, M_3, M_4, M_5, M_6, M_1, M_2, M_5)
      \end{scope}
    \end{tikzpicture}
  \end{minipage}
  }
  \subcaptionbox{\label{fig:prop_6:b}}{
    \begin{minipage}{.45\textwidth}  \centering
      \tikzstyle{EdgeStyle} = [->,>=stealth']%
      \label{fig:prop_6:b}
      \begin{tikzpicture}[scale=.35]
        \SetGraphUnit{3.5}
        \begin{scope}[execute at begin node=$, execute at end node=$]
          \Vertex{M_1} \EA(M_1){M_2} \EA(M_2){M_3} \EA(M_3){M_4} \EA(M_4){M_5} \EA(M_5){M_6}
          \NO(M_1){L_1} \NO(M_2){L_2} \NO(M_3){L_3} \NO(M_4){L_4} \NO(M_5){L_5} \NO(M_6){L_6} \EA(L_6){L_7}
          
          \Edges(L_1, M_1) \Edges(L_1, M_2)
          \Edges(L_2, M_2) \Edges(L_2, M_3)
          \Edges(L_3, M_3) \Edges(L_3, M_4)
          \Edges(L_4, M_4) \Edges(L_4, M_5)
          \Edges(L_5, M_5) \Edges(L_5, M_6)
          \Edges(L_6, M_6) \Edges(L_6, M_1)
          \Edges(L_7, M_2) \Edges(L_7, M_5)
        \end{scope}
      \end{tikzpicture}
    \end{minipage}
  }
  \caption{(a) example \(\dep(\mathbf{M})\) for which the minMCM (b) has 6 measurement variables and 7 latent variables}
\label{fig:prop_6}
\end{figure}

\begin{prop}
  \label{prop:more}
  There exist minMCMs containing more latent than measurement variables.
\end{prop}
\begin{proof}
  This follows from the graph theoretical characterization of minMCMs: there are at least as many latent variables as the intersection number of \(\dep(\mathbf{M})\), which in a graph with \(n\) vertices is (non-trivially) upper bounded by \(\frac{n^2}{4}\) \cite{Erdos_1966}.
  A simple example can be found when \(\dep(\mathbf{M})\) is as in Figure \ref{fig:prop_6}, resulting in \(n = 6\) nodes and an intersection number of \(i = 7\).
\end{proof}

\section{A minMCM-FINDING ALGORITHM AND ITS COMPLEXITY}
\label{sec:alg}

The procedure in the proof of Proposition \ref{prop:mlm-ecc} for constructing a minMCM from an undirected dependency graph leads directly to Algorithm \ref{alg:obs}. 

\IncMargin{1.5em}
\begin{algorithm}
  \SetKwInOut{Input}{Input}
  \SetKwInOut{Output}{Output}

\Indm
  \Input{ undirected dependency graph, \(\dep(\mathbf{M})\), over the measurement variables \(\mathbf{M}\)}
  \Output{ vertex-minimal or assignment-minimal MCM \(\mathcal{G}\) over \(\mathbf{M}\)}

  \Indp
  \BlankLine
  initialize edgeless graph with a vertex for each \(M \in \mathbf{M}\)\;
  find a clique-minimum or assignment-minimum edge clique cover of \(\dep(\mathbf{M})\), using the algorithm in Fig. 3 of \cite{Gramm_2009} or the algorithm \texttt{FIND-AM} of \cite{Ennis_2012}, respectively;\\
  \For{each clique \(C\) in the cover}{
    add vertex \(L\) with edges directed to each \(M \in C\);}
 \caption{constructing a minimal MeDIL causal model (minMCM)}
 \label{alg:obs}
\end{algorithm}
\DecMargin{1.5em}% \vspace{-3em}

Notice that Line 2 in Algorithm \ref{alg:obs} is to find a minimum ECC of \(\dep(\mathbf{M})\).
Nearly all of the computational complexity of Algorithm \ref{alg:obs} comes from this step, which is known to be an NP-hard problem, and so the choice of an efficient ECC-finding algorithm and implementation is especially important.

In case a clique-minimum ECC (and therefore vertex-minimum minMCM) is preferred, \cite{Gramm_2009} provides an exact algorithm.
The exact algorithm finds an ECC in \(\mathcal{O}(f(2^k) + n^4)\) time, where \(k\) is the number of cliques in the ECC and \(n\) is the number of vertices in the undirected graph, and is thus fixed-parameter tractable.
Furthermore, \cite{Cygan_2016} gives a lower bound on the complexity of the clique-minimum ECC problem and argues that the algorithm is probably optimal.
\citet{Gramm_2009} also provide a free/libre implementation of their algorithm, though it has not been maintained for some time and does not easily run on most modern machines.

In case an assignment-minimum ECC (and therefore edge-minimum minMCM) is preferred, \cite{Ennis_2012} provides an exact algorithm.
Though they do not offer an analyses of its complexity, it is essentially a backtracking algorithm based on \cite{Bron:1973}'s maximal clique finding algorithm, which has time complexity of\(\mathcal{O}(3^{n/3})\), and so this assignment-minimum ECC finding algorithm has an even larger complexity.
% \citet{Ennis_2012} do not provide a free/libre implementation of their algorithm.

As far as we are aware, no other implementations of the clique-minimum or assignment-minimum ECC finding algorithms exist.
To remedy this, we have implemented and released these and a few other related causal inference tools as a free/libre Python package at \href{https://medil.causal.dev}{https://medil.causal.dev}.
Already \citet{Gramm_2009} and \citet{Ennis_2012} showed that their algorithms perform in a reasonable amount of time on moderately sized graphs, e.g., returning a solution containing 100 cliques in a matter of minutes.
Unsurprisingly, given the hardware advancements of the past decade, our implementation performs even better, e.g. finding the 614 clique solution to the 61 node graph presented in the next section in only 39 seconds using an Intel Core i7-8700K CPU.

\section{APPLICATION}
\label{sec:application}
In this section we demonstrate the necessary steps to get from a raw data set to a minMCM output from our algorithm.
We then hint at how this output can be analyzed and suggest some conclusions that can be drawn from it.
Note that our contribution in this paper is theoretical, and the point of the following application is to make some of our theoretical claims and the potential use cases more concrete.

\subsection{THE DATA AND PREVIOUS ANALYSES}
\label{sec:data-prev-analys}

The \textit{Stress, Religious Coping, and Depression} data set\footnote{We would like to thank David Danks and especially Joseph Ramsey at Carnegie Mellon University for providing us with a copy.}
was collected by Bongjae Lee from the University of Pittsburgh in 2003.
There were 127 participants answering a total of 61 questions: 21 designed to measure stress, 20 for religious coping, and 20 for depression---see \cite{silva2005generalized} for the full questionnaire.
This data has been analyzed by several other measurement model methods \cite{silva2005generalized,Silva:2006,Kummerfeld:2014,Kummerfeld:2016}, and their findings (which largely agree with each other) can be briefly summarized as follows:
(i) in contrast to the design goal, most of the measurement variables are ``impure'' in that they are caused by multiple latent variables; % not quite right, but close enough?
(ii) they are able to find some subsets (ranging in number from three to nine)  of ``pure'' measurement variables that passed their significance tests and some of which suggest a model similar to what Lee hypothesized containing three latent variables---the first of which causes only measurement variables of stress, the second only depression, and the third only coping;
(iii) most of their models scoring the highest significance are more complex models than Lee's model (the most complex containing eight latents \cite{silva2005generalized}).

\subsection{ANALYSIS USING minMCMS}
\label{sec:analys-using-minmcms}
Notice that the input to Algorithm \ref{alg:obs} is an undirected dependency graph, while in practice one does not have direct knowledge of the (in)dependencies themselves but only samples of the measurement variables.
It is therefore necessary to first estimate the independencies before applying this algorithm.
Because the algorithm is agnostic to the test statistic, it is not constrained to linear methods such as Pearson correlation (for which ``\(X \indep Y \implies \mathrm{corr}(X, Y) = 0\)'' but not the converse) but can leverage the power of nonlinear independence tests \cite{Gretton_2005, Szekely_2007}.
We used the distance correlation \cite{Szekely_2007} as our test statistic (with the property ``\(X \indep Y \iff \mathrm{dCorr}(X, Y) = 0\)'') and performed 1000 random permutations of the measurement variables to sample from the null-distribution \cite{Dwass_1957}.
The \(p\)-value for each pair was then calculated as the proportion of the permutation tests in which the absolute distance correlation of the pair of variables with permuted samples exceeded that of the original pair.
Finally, independence between two variables was concluded if the distance correlation between them was less than 0.1 and the corresponding \(p\)-value was greater than 0.1.%
% \mgw{}{[should we insert a sentence here in how far the results depend on that (somewhat arbitrary) threshold?]}. -----maybe the related comment in the discussion section is sufficient, or a better place?
\footnote{As one would expect, using a nonlinear measure of dependence allows us to detect more dependencies: we found almost 31\% of the over 1500 estimated nonlinear pairwise dependencies (i.e., edges in the UDG) to be undetectable using the linear Pearson correlation.}

The binary-valued \(61\times 61\) matrix corresponding to the estimated independencies, with a 0 for independence and 1 for dependence thus forms the adjacency matrix for the UDG that is input for Algorithm \ref{alg:obs}.
We decided to find a latent-minimal minMCM, and the result has 614 latent variables.
It is thus too complex to be legibly displayed here, so we instead present % a part of its incidence matrix (Figure \ref{fig:incidence}) along with
figures \ref{fig:histograms} and \ref{fig:heatmaps} to facilitate analysis of the results.

\begin{figure}[t]
  \centering
  % \subcaptionbox{\label{fig:histograms:a}}{}
  % \subcaptionbox{\label{fig:histograms:b}}{}
      \includegraphics[width=\linewidth]{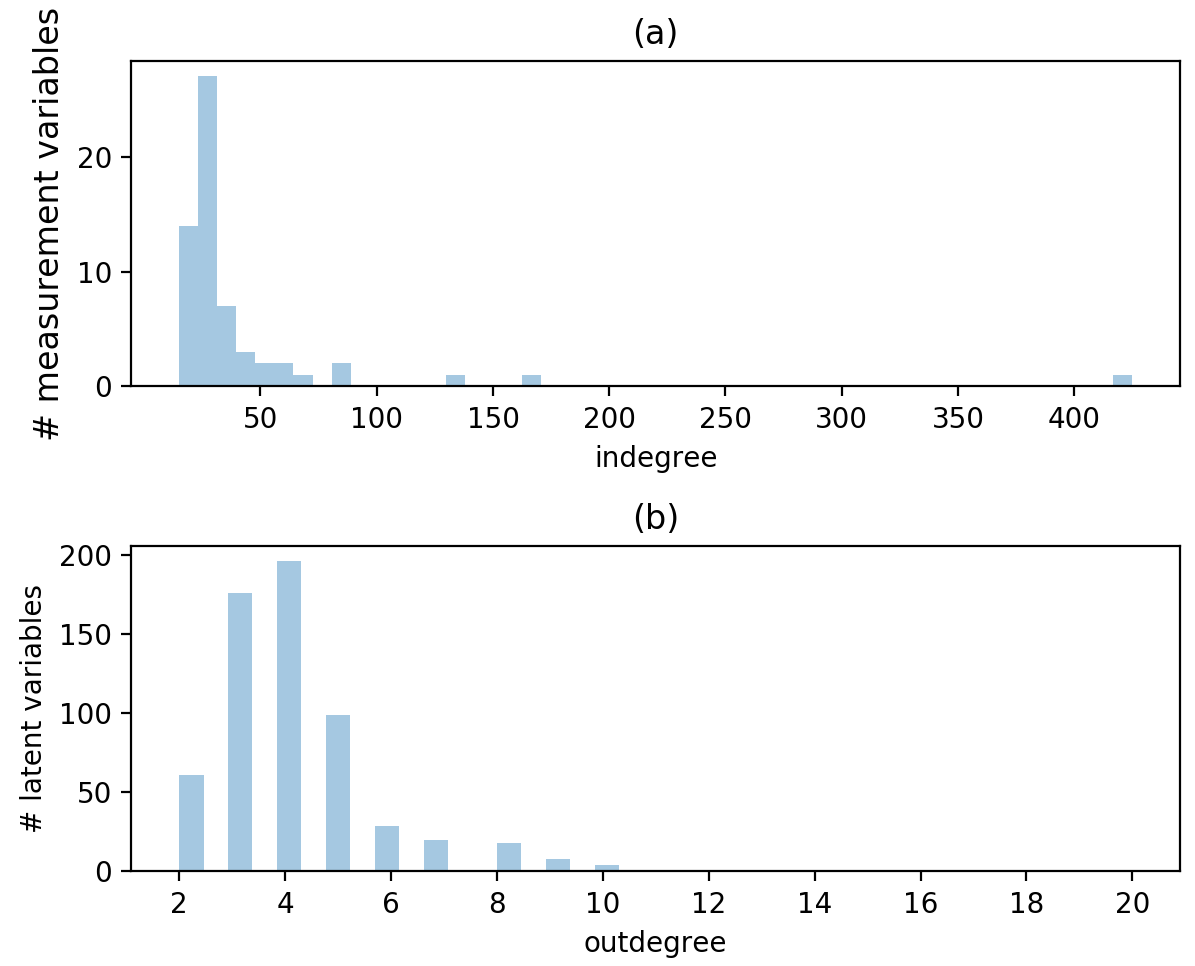}
  \caption{histograms showing (a) indegree of the measurement variables and (b) the outdegree of the latent variables}
\label{fig:histograms}
\end{figure}

\begin{figure}[t]
  \centering
  % \subcaptionbox{\label{fig:heatmaps:a}}{}
  % \subcaptionbox{\label{fig:heatmaps:b}}{}
  \includegraphics[width=.8\linewidth]{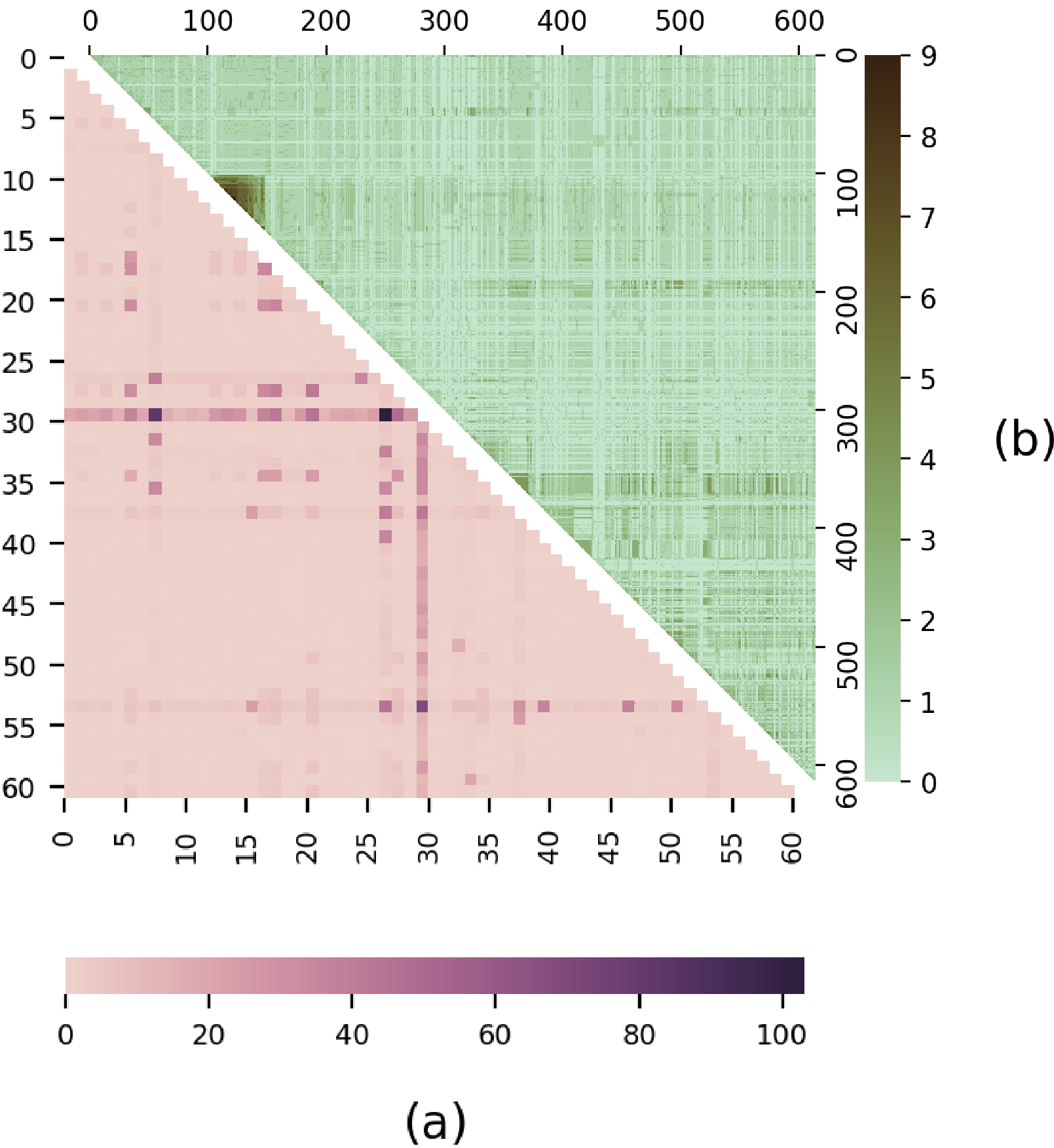}  
  \caption{heatmaps showing (a) the number of latent variables each pair of the 61 measurement variables have in common and (b) the number of measurement variables each pair of the 614 latent variables have in common}
  \label{fig:heatmaps}
\end{figure}

Looking at the histogram in Figure \ref{fig:histograms}(a), we find a median indegree (i.e., number of latent causes) of the measurement variables of 27, but with one in particular, \(M_{30}\), having 425.
The item in the questionnaire corresponding to \(M_{30}\) was the ninth in the set designed to measure depression, and it asked participants how frequently the event ``I thought my life had been a failure'' occurred in the preceding week.
Semantically, it makes sense that this item would have many more latent causes than the other items, because its scope is much larger, requiring reflection on the participants' entire life up to that point instead of just during the week in question, as is the case for other depression items, such as ``I enjoyed life'' (\(M_{37}\), 24 latent causes) and ``I felt sad'' (\(M_{39}\), 25 latent causes).
Furthermore, looking at Figure \ref{fig:heatmaps}(a), showing the number of latents each pair of measurement variables share, we see that \(M_{30}\) shares a relatively high amount of latent causes with the other measurement variables (median of 21), while for \(M_{37}\) and \(M_{39}\) the median of shared latent causes is one. 
Our analysis thus agrees with the previous analyses described in Section \ref{sec:data-prev-analys} insofar as we also find many ``impure'' measurement variables, but extends their insights by differentiating between measurement variables that are best considered a general or mixed measurement ($M_{30}$) and those that, even though they are also impure, span different subsets of the latent space ($M_{37}$ and $M_{39}$).

Looking at the outdegree (i.e., the number of measurement variables a latent causes) in Figure \ref{fig:histograms}(b) we find a median of four and a range from 2 to 20. The number of measurements shared by each pair of latent variables reveals further structure (Figure \ref{fig:heatmaps}(b)). In particular, the incidence matrix representation of the latents corresponding to the block structure between approximately \(L_{105}\)--\(L_{145}\) reveals seven measurement variables that these latents mostly have in common, corresponding to four stress and three depression items. On the other side, 41\% (roughly 74k) pairs of variables do not share any measurement variables. Such insights may be used to simplify models, e.g.~by removing measurement variables that induce multiple latents, or to build subsets of ``pure'' measurement variables, in the sense that the resulting measurement subsets are caused by disjoint sets of latents\footnote{Note that this is a bit different from the notion of "pure" used in the other measurement literature}.

Finally, we note that there is more structure to be explored in the minMCM and figures \ref{fig:histograms} and \ref{fig:heatmaps} % (especially if we were to further use methods from network analysis)
, but that is beyond our present scope.
Note that the type of structure analyzed here emerges only when considering an ECC (i.e. patterns in the UDG, which is an abstraction of the correlation matrix) and not from the correlation matrix itself---analogous to higher-moment statistics or higher-order logic.

Our findings are not inconsistent with previous analyses of this data set, as can be seen by their agreement with points (i) and (iii) in Section \ref{sec:data-prev-analys}, and should rather be seen as complementary.
More generally our algorithm and corresponding analyses do not subsume existing methods but rather provide a novel perspective that allows us to focus on otherwise unutilized structure in measurement data, which in addition to helping to model the data also aids in, e.g., assessing and revising questionnaires and instruments.

\section{DISCUSSION}
\label{sec:discussion-future-work}
Having in the preceding sections presented our minMCM finding algorithm, its supporting theory, and a demonstration application, we now conclude with two main directions for future work:
the first direction is primarily concerned with applications of Algorithm \ref{alg:obs} in its current state or requiring only minor modifications,
while the second is primarily concerned with significantly extending Algorithm \ref{alg:obs} and with developing new methods based on insights gleaned during its development.

\subsection{FUTURE APPLICATIONS AND MINOR MODIFICATIONS}
\label{sec:future-appl-minor}

% probably remove some of this?
Being constraint-based, the Algorithm \ref{sec:alg} relies on estimated independencies.
Thus, errors in the inference of minMCMs come not from Algorithm \ref{alg:obs} itself but rather from the estimation of independencies that it (along with many other causal inference methods) requires as input.
In this regard, a single incorrectly estimated independence can in the (unlikely) worst case\footnote{This is when the inclusion/exclusion of a single edge in an \(n \geq 3\) vertex undirected dependency graph makes the difference between the graph having \(2(n-2)\) maximal cliques that are all edges and \(n-2\) maximal cliques that are all triangles. Fortunately, such precarious structures are easy to detect and can be removed by picking different sets of measurements.}%
result in incorrectly doubling or halving the number of estimated latents or edges.
In any case, as mentioned at the end of Section \ref{sec:obs_mcm}, further estimates of conditional independencies can help corroborate or refute the estimated unconditional independencies.
More detailed examination is needed to make this more theoretically precise as well as to determine how much of a problem this is likely to pose for real data.

One final caveat for interpreting minMCMs is that, for complex graphs, there can be multiple minimum ECCs (for both types of minimality), each with the same minimum number of cliques or assignments.
Thus, while using a minMCM to reason about the minimum number of edges or latents is always valid, stronger conclusions may require that the graph \(\dep(\mathbf{M})\) admits only one minMCM (which is simple enough to test) or that further assumptions or background knowledge are used to justify one minMCM over other observationally consistent ones.
To this end, the (non-minimal) MCM corresponding to maximal cliques (e.g., Figure \ref{fig:diff:b}) may be especially interesting, because it contains all observationally consistent MCMs (including the minMCMs in \ref{fig:diff:c} and \ref{fig:diff:d}).

Another promising aspect of our approach for future work is its extensibility, which results from establishing MeDIL causal models as a causal semantics for edge clique covers.
Though we have so far focused on minimal ECCs, a MCM corresponding to \textit{any} ECC for a given UDG is guaranteed to be measurement-faithful and causally sufficient (though not minimal or measurement-Markov) for the corresponding distribution of measurement variables.
Using a different class of ECCs simply requires a different algorithm to be used in Line 2 of Algorithm \ref{alg:obs}.
Just as we expressed simplicity of the causal model in terms of the number of latents (or edges) in the MCM and therefore the number of cliques (or assignments) in the ECC, \textit{any} property of a causal model that can be expressed in terms of properties of an ECC can be used to repurpose an ECC-finding algorithm for the desired CSL task.
For example, developments in network science \cite{Conte_2019} make it possible for ECC-based causal analysis of very large graphs, even containing up to millions of nodes.

\subsection{EXTENSIONS AND FURTHER DEVELOPMENTS}
\label{sec:sign-extens-furth}

Because Algorithm \ref{alg:obs} returns a causally sufficient DAG, it should be possible to actually learn a corresponding fully specified functional causal model using, e.g., some version of nonlinear ICA or variational autoencoders \cite{khemakhem2019variational} that has been modified to take into account the conditional independence structure.
This could potentially lead to the development of a causal, non-parametric generalization of factor analysis \cite{martin1994discrete} which would still be interestingly different from similar existing work \cite{hoyer2008estimation,Kummerfeld:2016}.
Furthermore, since learning such a FCM would require the data set and not just its CI relations, it would be straight-forward to make a score-based adaptation of Algorithm \ref{alg:obs} inspired by \cite{elidan2001discovering}, where cliques are picked according to maximizing a scoring criterion instead of (possibly misestimated) CI relations.
This would help overcome the potential pitfall mentioned in Section \ref{sec:future-appl-minor}.

Additionally, notice that formally, (though not semantically) \textit{every} DAG is a MCM:
any given DAG \(\mathcal{G}\) can be partitioned into sink nodes \(\mathbf{S}\) and non-sink nodes \(\mathbf{N}\), in which case it is observationally consistent with respect to \(\mathbf{S}\) to any other DAG \(\mathcal{H}\) whose (sub)set of sink nodes \(\mathbf{S'}\) has the same UDG as \(\mathbf{S}\).
This allows for some of the theory developed in sections \ref{sec:obs_mcm} and \ref{sec:ecc} to be easily repurposed to characterizing subset-Markov equivalence classes for DAGs with different sets of variables, as long as they have some subset of sink nodes \(\mathbf{S}=\mathbf{S'}\) in common.
This may help connect causal coarsening \cite{Chalupka_2016} with causally consistent transformations between micro- and macro-models \cite{Rubenstein:2017} and causal abstraction \cite{beckers2019abstracting}.

% In cases where \(|\mathbf{N}|\) is greater than \(\mathbf{S}\), the nodes \(\mathbf{N}\) can constitute a causally coarsened representation of \(\mathbf{S}\) or macromodel or abstraction
% However, because MCMs can also in some cases require fewer \(|\mathbf{N}|\) than \(\mathbf{S}\), they provide some insight into a ``refined'' (instead of coarsened) representation or micromodel or  

\vspace{-1em}
\renewcommand{\bibsection}{\subsubsection*{References}} % uai2020.sty is poorly written
\bibliographystyle{apalike}
\bibliography{mcm}
\end{document}